\documentclass{article}

\PassOptionsToPackage{compress, numbers, sort}{natbib}

\usepackage[final]{neurips_2019}

\usepackage[utf8]{inputenc} 
\usepackage[T1]{fontenc}    
\usepackage{hyperref}       
\usepackage{url}            
\usepackage{booktabs}       
\usepackage{amsfonts}       
\usepackage{nicefrac}       
\usepackage{microtype}      
\usepackage{graphicx}
\usepackage{subfigure}
\usepackage{paralist}
\usepackage{amsmath}
\usepackage{color}
\usepackage{amsthm}
\usepackage{chngcntr}
\usepackage{comment}
\usepackage{tabularx}
\usepackage{siunitx} 
\usepackage{etoolbox}

\robustify\bfseries

\usepackage{hyperref}
\usepackage{algorithm}
\usepackage{algorithmic}
\usepackage{forloop}

\DeclareMathOperator{\degree}{deg}

\DeclareMathOperator{\hash}{hash}
\DeclareMathOperator{\concatenate}{concatenate}
\DeclareMathOperator{\pairwisedist}{pairwise\_dist}
\DeclareMathOperator{\CAT}{CAT}
\DeclareMathOperator*{\argmin}{arg\,min}

\newtheorem{definition}{Definition}
\newtheorem{proposition}{Proposition}
\newtheorem{theorem}{Theorem}
\newtheorem{lemma}{Lemma}
\newtheorem{conjecture}{Conjecture}

\definecolor{mydarkblue}{rgb}{0,0.08,0.45}

\title{Wasserstein Weisfeiler-Lehman Graph Kernels}

\author{%
  Matteo Togninalli$^{1, 2, *}$ \\
  \small\texttt{matteo.togninalli@bsse.ethz.ch} \\
  \And
  Elisabetta Ghisu$^{1, 2, *}$ \\
  \small\texttt{elisabetta.ghisu@bsse.ethz.ch} \\
  \And
  Felipe Llinares-L\'{o}pez$^{1, 2}$ \\
  \small\texttt{felipe.llinares@bsse.ethz.ch} \\
  \And
  Bastian Rieck$^{1, 2}$\\
  \small\texttt{bastian.rieck@bsse.ethz.ch} \\
  \And
  Karsten Borgwardt$^{1, 2}$ \\
  \small\texttt{karsten.borgwardt@bsse.ethz.ch} \\\\
  \small\textsc{$^{1}$Department of Biosystems Science and Engineering, ETH Zurich, Switzerland}\\
    \small\textsc{$^{2}$SIB Swiss Institute of Bioinformatics, Switzerland}\\
    \small{$^{*}$These authors contributed equally}\\
}

\begin{document}

\maketitle

\begin{abstract}
Most graph kernels are an instance of the class of \mbox{$\mathcal{R}$-Convolution} kernels, which measure the similarity of objects by comparing their substructures.
Despite their empirical success, most graph kernels use a naive aggregation of the final set of substructures, usually a sum or average, thereby potentially discarding valuable information about the distribution of individual components. Furthermore, only a limited instance of these approaches can be extended to continuously attributed graphs. 
We propose a novel method that relies on the Wasserstein distance between the node feature vector distributions of two graphs, which allows finding subtler differences in data sets by considering graphs as high-dimensional objects  rather than simple means.
We further propose a Weisfeiler--Lehman-inspired embedding scheme for graphs with continuous node attributes and weighted edges, enhance it with the computed Wasserstein distance, and thereby improve the state-of-the-art prediction performance on several graph classification tasks.
\end{abstract}

\section{Introduction}
\label{introduction}

Graph-structured data have become ubiquitous across domains over the last decades, with examples ranging from social and sensor networks to chemo- and bioinformatics. 
Graph kernels~\citep{vishwanathan2010graph} have been highly successful in dealing with the complexity of graphs and have shown good predictive performance on a variety of classification problems~\citep{shervashidze2011weisfeiler,morris2016faster, yanardag2015deep}. 
Most graph kernels rely on the \mbox{$\mathcal{R}$-Convolution} framework~\citep{haussler1999convolution}, which decomposes structured objects into substructures to compute local similarities that are then aggregated. Although being successful in several applications, \mbox{$\mathcal{R}$-Convolution} kernels on graphs have limitations:
\begin{inparaenum}[(1)]
 \item the simplicity of the way in which the similarities between substructures are aggregated might limit their ability to capture complex characteristics of the graph;
 \item most proposed variants do not generalise to graphs with high-dimensional continuous node attributes, and extensions are far from being straightforward.
\end{inparaenum}
Various solutions have been proposed to address point~(1). For example, \citet{frohlich2005} introduced kernels based on the optimal assignment of node labels for molecular graphs, although these kernels are not positive definite~\citep{vert2008optimal}. Recently, another approach was proposed by \citet{kriege2016oa}, which employs a Weisfeiler--Lehman based colour refinement scheme and uses an optimal assignment of the nodes to compute the kernel. However, this method cannot handle continuous node attributes, leaving point~(2) as an open problem.
 
To overcome both limitations, we propose a method that combines the most successful vectorial graph representations derived from the graph kernel literature with ideas from optimal transport theory, which have recently gained considerable attention. 
In particular, improvements of the computational strategies to efficiently obtain Wasserstein distances~\citep{cuturi2013sinkhorn, altschuler2017near} have led to many applications in machine learning that use it for various purposes, ranging from generative models~\citep{arjovsky2017wasserstein} to new loss functions~\citep{frogner2015learning}. 
In graph applications, notions from optimal transport were used to tackle the graph alignment problem \citep{xu2019gromov}.
In this paper, we provide the theoretical foundations of our method, define a new graph kernel formulation, and present successful experimental results. Specifically, our main contributions can be summarised as follows:
\begin{compactitem}
    \item We present the graph Wasserstein distance, a new distance between graphs based on their node feature representations, and we discuss how kernels can be derived from it.
    \item We introduce a Weisfeiler--Lehman-inspired embedding scheme that works for both categorically labelled and continuously attributed graphs, and we couple it with our graph Wasserstein distance;
    \item We outperform the state of the art for graph kernels on traditional graph classification benchmarks with continuous attributes.
\end{compactitem}

\section{Background: graph kernels and Wasserstein distance}
\label{sec:background}

In this section, we introduce the notation that will be used throughout the manuscript. Moreover, we provide the necessary background on graph kernel methods and the Wasserstein distance.

\subsection{Graph kernels}

Kernels are a class of similarity functions that present attractive properties to be used in learning algorithms~\citep{scholkopf2002learning}.
Let $\mathcal{X}$ be a set and $k\colon \mathcal{X} \times \mathcal{X} \to \mathbb{R}$ be a function associated with a Hilbert space $\mathcal{H}$, such that there exists a map $\phi\colon \mathcal{X} \to \mathcal{H}$ with $k(x,y) = \left \langle \phi(x), \phi(y)\right \rangle_{\mathcal{H}}$. Then, $\mathcal{H}$ is a reproducing kernel Hilbert space~(RKHS) and $k$ is said to be a positive definite kernel. 
A positive definite kernel can be interpreted as a dot product in a high-dimensional space,
thereby permitting its use in any learning algorithm that relies on dot products, such as support vector machines~(SVMs), by virtue of the \emph{kernel trick}~\citep{scholkopf2001kernel}. Because ensuring positive definiteness is not always feasible, many learning algorithms were recently proposed to extend SVMs to indefinite kernels~\citep{ong2004learning, balcan2008theory, loosli2015learning, oglic2018learning}.

We define a graph as a tuple $G = (V,E)$, where $V$ and $E$ denote the set of nodes and edges, respectively; we further assume that the edges are undirected. Moreover, we denote the cardinality of nodes and edges for $G$ as  $|V|=n_G$ and $|E| = m_G$. 
For a node $v\in V$, we write $\mathcal{N}(v) = \{u\in V \mid (v,u)\in E\}$ and $|\mathcal{N}(v)| = \degree(v)$ to denote its first-order neighbourhood. We say that a graph is \emph{labelled} if its nodes have categorical labels. A label on the nodes is a function $l\colon V\to \Sigma$ that assigns to each node $v$ in $G$ a value $l(v)$ from a finite label alphabet $\Sigma$. 
Additionally, we say that a graph is \emph{attributed} if for each node $v\in V$ there exists an associated vector $a(v)\in\mathbb{R}^{m}$. In this paper, $a(v)$ are the node attributes and $l(v)$ are the categorical node labels of node $v$. In particular, the node attributes are high-dimensional continuous vectors, whereas the categorical node labels are assumed to be integer numbers (encoding either an ordered discrete value or a category). With the term ``node labels'', we will implicitly refer to categorical node labels.
Finally, a graph can have weighted edges, and the function $w\colon E\to\mathbb{R}$ defines the weight $w(e)$ of an edge $e := (v,u) \in E$. 

Kernels on graphs are generally defined using the \mbox{$\mathcal{R}$-Convolution} framework by \cite{haussler1999convolution}. The main idea is to decompose graph $G$ into substructures and to define a kernel value $k(G,G')$ as a combination of substructure similarities. A pioneer kernel on graphs was presented by~\citep{kashima2003}, where node and edge attributes are exploited for label sequence generation using a random walk scheme. Successively, a more efficient approach based on shortest paths~\citep{borgwardt2005shortest} was proposed, which computes each kernel value $k(G,G')$ as a sum of the similarities between each shortest path in $G$ and each shortest path in $G'$. Despite the practical success of \mbox{$\mathcal{R}$-Convolution} kernels, they often rely on aggregation strategies that ignore valuable information, such as the distribution of the substructures. An example is the Weisfeiler--Lehman~(WL) subtree kernel or one of its variants~\citep{shervashidze2009FastSK, shervashidze2011weisfeiler, Rieck19b}, which generates graph-level features by summing the contribution of the node representations. To avoid these simplifications, we want to use concepts from optimal transport theory, such as the Wasserstein distance, which can help to better capture the similarities between graphs.  

\subsection{Wasserstein distance}

The Wasserstein distance is a distance function between probability distributions defined on a given metric space. Let $\sigma$ and $\mu$ be two probability distributions on a metric space $M$ equipped with a ground distance $d$, such as the Euclidean distance.
\begin{definition}
\label{def:wasserstein}
The $L^p$-Wasserstein distance for $p \in \left [1, \infty \right )$ is defined as
\begin{equation}
\label{eq:wass_definition}
    W_p(\sigma, \mu) := \left (  \inf_{\gamma \in \Gamma(\sigma, \mu)} {\displaystyle \int_{M\times M} d(x,y)^p \operatorname{d}\!\gamma(x,y)} \right )^{\frac{1}{p}},
\end{equation}
where $\Gamma(\sigma, \mu)$ is the set of all transportation plans $\gamma \in \Gamma(\sigma, \mu)$ over $M \times M$ with marginals $\sigma$ and $\mu$ on the first and second factors, respectively.
\end{definition}
The Wasserstein distance satisfies the axioms of a metric, provided that $d$ is a metric~(see the monograph of \citet{villani2008optimal}, chapter 6, for a proof).
Throughout the paper, we will focus on the distance for $p = 1$ and we will refer to the $L^1$-Wasserstein distance when mentioning the Wasserstein distance, unless noted otherwise.

The Wasserstein distance is linked to the optimal transport problem \citep{villani2008optimal}, where the aim is to find the most ``inexpensive'' way, in terms of the ground distance, to transport all the probability mass from distribution $\sigma$ to match distribution $\mu$. 
An intuitive illustration can be made for the $1$-dimensional case, where the two probability distributions can be imagined as piles of dirt or sand. The Wasserstein distance, sometimes also referred to as the earth mover's distance~\citep{rubner2000earth}, can be interpreted as the minimum effort required to move the content of the first pile to reproduce the second pile.

In this paper, we deal with finite sets of node embeddings and not with continuous probability distributions. Therefore, we can reformulate the Wasserstein distance as a sum rather than an integral, and use the matrix notation commonly encountered in the optimal transport literature~\citep{rubner2000earth} to represent the transportation plan.
Given two sets of vectors $X \in \mathbb{R}^{n \times m}$ and $X'\in \mathbb{R}^{n' \times m}$, we can equivalently define the Wasserstein distance between them as
\begin{equation}
\label{eq:wassdiscreteformulation}
    W_1(X,X') :=  \min_{P \in \Gamma(X, X')} \left \langle P, M\right \rangle.
\end{equation}
Here, $M$ is the distance matrix containing the distances $d(x,x')$ between each element $x$ of $X$ and $x'$ of $X'$,  $P \in \Gamma$ is a transport matrix (or joint probability), and $\langle \cdot, \cdot \rangle$ is the Frobenius dot product. The transport matrix $P$ contains the fractions that indicate how to transport the values from $X$ to $X'$ with the minimal total transport effort. Because we assume that the total mass to be transported equals $1$ and is evenly distributed across the elements of $X$ and $X'$, the row and column values of $P$ must sum up to $\nicefrac{1}{n}$ and $\nicefrac{1}{n'}$, respectively.


\section{Wasserstein distance on graphs}
\label{sec:GWD}

The unsatisfactory nature of the aggregation step of current \mbox{$\mathcal{R}$-Convolution} graph kernels, which may mask important substructure differences by averaging, motivated us to have a finer distance measure between structures and their components. 
In parallel, recent advances in optimisation solutions for faster computation of the optimal transport problem inspired us to consider this framework for the problem of graph classification.
Our method relies on the following steps:
\begin{inparaenum}[(1)]
 \item transform each graph into a set of node embeddings,
 \item measure the Wasserstein distance between each pair of graphs, and
 \item compute a similarity matrix to be used in the learning algorithm.
\end{inparaenum}
Figure~\ref{fig:scheme} illustrates the first two steps, and 
Algorithm~\ref{alg:gwk} summarises the whole procedure. We start by defining an embedding scheme and illustrate how we integrate embeddings in the Wasserstein distance.

\begin{definition}[Graph Embedding Scheme]
\label{def:embed_scheme}
Given a graph $G=(V,E)$, a graph embedding scheme $f\colon \mathit{G} \to \mathbb{R}^{|V| \times m},\; f(G) = X_G$ is a function that outputs a fixed-size vectorial representation for each node in the graph. For each $v_i\in V$, the $i$-th row of $X_G$ is called the node embedding of $v_i$.
\end{definition}

Note that Definition~\ref{def:embed_scheme} permits treating node labels, which are categorical attributes, as one-dimensional attributes with $m=1$.

\begin{definition}[Graph Wasserstein Distance]
\label{def:graph_wass}
Given two graphs $G=(V,E)$ and $G'=(V',E')$, a graph embedding scheme $f\colon \mathit{G} \rightarrow \mathbb{R}^{|V| \times m} $ and a ground distance $d\colon \mathbb{R}^{m}\times \mathbb{R}^{m} \rightarrow \mathbb{R}$, we define the Graph Wasserstein Distance~(GWD) as
\begin{equation}
\label{eq:graph_wass}
D^f_{W}(G,G') := W_1(f(G),f(G')).
\end{equation}
\end{definition}

We will now propose a graph embedding scheme inspired by the WL kernel on categorically labeled graphs, extend it to continuously attributed graphs with weighted edges, and show how to integrate it with the GWD presented in Definition \ref{def:graph_wass}. 

\begin{figure*}[t]
\vskip 0.2in
  \centering
  \includegraphics[width=0.65\columnwidth]{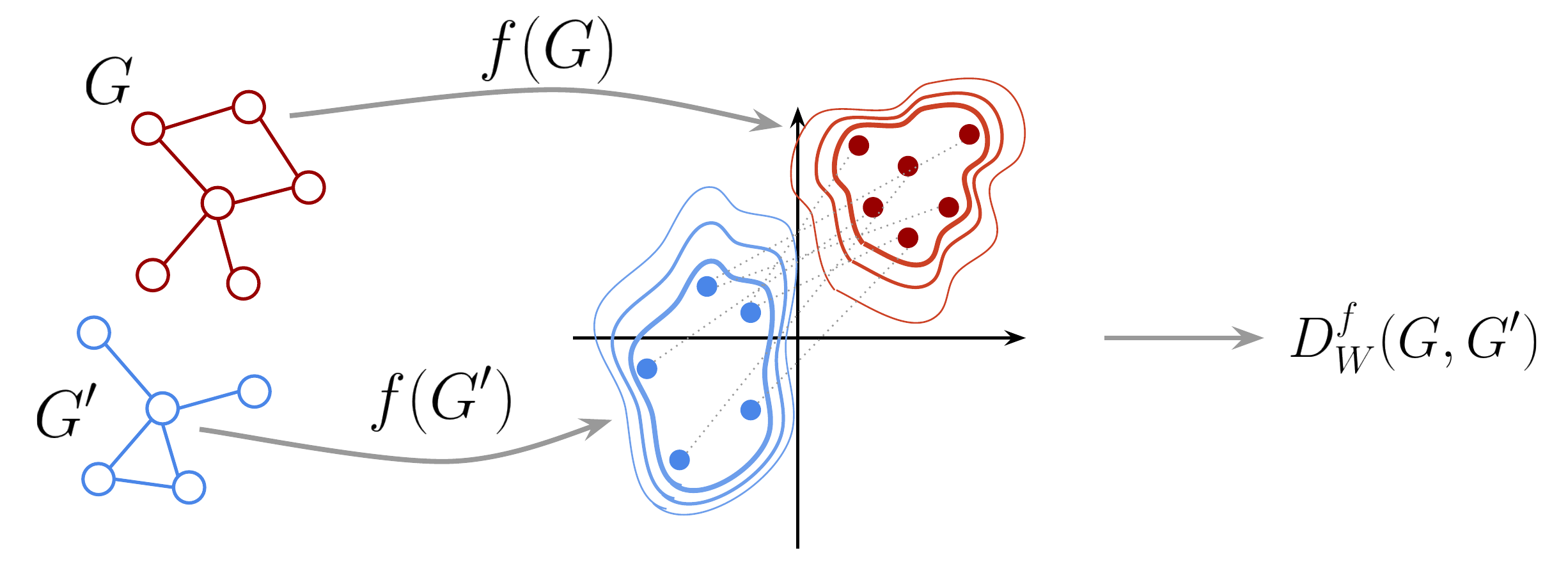}
\caption{Visual summary of the graph Wasserstein distance. First, $f$ generates embeddings for two input graphs $G$ and $G'$. Then, the Wasserstein distance between the embedding distributions is computed.}
\label{fig:scheme}
\vskip -0.2in
\end{figure*}

\subsection{Generating node embeddings} \label{sec:NodeEmbeddings}

\textbf{The Weisfeiler--Lehman scheme.} The Weisfeiler--Lehman subtree kernel~\citep{shervashidze2009FastSK, shervashidze2011weisfeiler}, designed for labelled non-attributed graphs, looks at similarities among subtree patterns, defined by a propagation scheme on the graphs that iteratively compares labels on the nodes and their neighbours. This is achieved by creating a sequence of ordered strings through the aggregation of the labels of a node and its neighbours; those strings are subsequently \textit{hashed} to create updated \textit{compressed} node labels.
With increasing iterations of the algorithm, these labels represent increasingly larger neighbourhoods of each node, allowing to compare more extended substructures. 

Specifically, consider a graph $G = (V, E)$, let $\ell^0(v) = \ell(v)$ be the initial node label of $v$ for each $v\in V$, and let $H$ be the number of WL iterations. Then, we can define a recursive scheme to compute $\ell^{h}(v)$ for $h=1,\ldots, H$ by looking at the ordered set of neighbours labels $\mathcal{N}^{h}(v) = \{\ell^{h}(u_0),\ldots,\ell^{h}(u_{\degree(v)-1})\}$ as 
\begin{equation}
    \ell^{h+1}(v) = \hash(\ell^{h}(v), \mathcal{N}^h(v)).
\end{equation}

We call this procedure the WL labelling scheme. As in the original publication~\citep{shervashidze2009FastSK}, we use perfect hashing for the $\hash$ function, so nodes at iteration $h+1$ will have the same label if and only if their label and those of their neighbours are identical at iteration $h$. 

\textbf{Extension to continuous attributes.} For graphs with continuous attributes $a(v)\in\mathbb{R}^{m}$, we need to improve the WL refinement step, whose original definition prohibits handling the continuous case. The key idea is to create an explicit propagation scheme that leverages and updates the current node features by averaging over the neighbourhoods.
Although similar approaches have been implicitly investigated for computing node-level kernel similarities~\citep{neumann2016, morris2016faster}, they rely on additional hashing steps for the continuous features. 
Moreover, we can easily account for edge weights by considering them in the average calculation of each neighbourhood.
Suppose we have a continuous attribute $a^{0}(v) = a(v)$ for each node $v \in G$. Then, we recursively define
\begin{equation}
a^{h+1}(v) =  \frac{1}{2}\left(a^{h}(v) + \frac{1}{\degree(v)}  \sum_{u\in \mathcal{N}(v)} w\left(\left(v,u\right)\right)\cdot a^{h}(u)\right).
\end{equation}
When edge weights are not available, we set $w(u, v) = 1$.
We consider the weighted average of the neighbourhood attribute values instead of a sum and add the $\nicefrac{1}{2}$ factor because we want to ensure a similar scale of the features across iterations; in fact, we  concatenate such features for building our proposed kernel~(see Definition~\ref{def:dwl_embedding} for more details) and observe better empirical results with similarly scaled features. 
Although this is not a test of graph isomorphism, this refinement step can be seen as an intuitive extension for continuous attributes of the one used by the WL subtree kernel on categorical node labels, a widely successful baseline. 
Moreover, it resembles the propagation scheme used in many graph neural networks, which have proven to be successful for node classification on large data sets~\citep{duvenaud2015convolutional, kipf2017semi, klicpera2018combining}. 
Finally, its ability to account for edge weights makes it applicable to all types of graphs without having to perform a hashing step~\citep{morris2016faster}. 
Further extensions of the refinement step to account for high-dimensional edge attributes are left for future work. A straightforward example would be to also apply the scheme on the \emph{dual} graph~(where each edge is represented as a node, and connectivity is established if two edges in the primal graph share the same node) to then combine the obtained kernel with the kernel obtained on primal graphs via appropriate weighting.

\textbf{Graph embedding scheme.} Using the recursive procedure described above, we propose a WL-based graph embedding scheme that generates node embeddings from the node labels or attributes of the graphs. In the following, we use $m$ to denote the dimensionality of the node attributes~($m = 1$ for the categorical labels).

\begin{definition}[WL features]
\label{def:dwl_embedding}
Let $G = (V,E)$ and let $H$ be the number of WL iterations. Then, for every $h \in \{0,\ldots,H\}$, we define the WL features as 
\begin{equation}
 X^h_G = [x^h(v_1),\ldots,x^h(v_{n_G})]^{T},
\end{equation}
where $x^h(\cdot)=\ell^h(\cdot)$ for categorically labelled graphs and $x^h(\cdot)=a^h(\cdot)$ for continuously attributed graphs. We refer to $X^h_G\in \mathbb{R}^{n_{G}\times m}$ as the \emph{node features} of graph $G$ at iteration $h$. Then, the node embeddings of graph $G$ at iteration $H$ are defined as
\begin{equation}
\begin{split}\label{embed}
f^H\colon G &\to \mathbb{R}^{n_G\times (m(H+1))}\\
G &\mapsto \concatenate(X^0_G, \ldots, X^H_G).
\end{split}
\end{equation}

\end{definition}

We observe that a graph can be both \textit{categorically labelled} and \textit{continuously attributed}, and one could extend the above scheme by jointly considering this information (for instance, by concatenating the node features).
However, we will leave this scenario as an extension for future work; thereby, we avoid having to define an appropriate distance measure between categorical and continuous data, as this is a long-standing issue \citep{stevens1946theory}.

\subsection{Computing the Wasserstein distance}

Once the node embeddings are generated by the graph embedding scheme, we evaluate the pairwise Wasserstein distance between graphs. We start by computing the ground distances between each pair of nodes. For categorical node features, we use the normalised Hamming distance:
\begin{equation}
\label{eq:hammingdist}
d_{\text{Ham}}(v, v') = \frac{1}{H+1}\sum_{i=1}^{H+1} \rho (v_{i}, v'_{i}), \; \; 
\rho(x, y)=\left\{\begin{array}{l}
1, \; x \neq y
\\ 
0, \; x = y
\end{array}\right.
\end{equation}

The Hamming distance can be pictured as the normalised sum of discrete metric $\rho$ on each of the features.
The Hamming distance equals $1$ when two vectors have no features in common and $0$ when the vectors are identical. 
We use the Hamming distance as, in this case, the Weisfeiler--Lehman features are indeed categorical, and values carry no meaning.
For continuous node features, on the other hand, we employ the Euclidean distance: 
\begin{equation}
    d_E(v,v')=||v-v'||_2.
\end{equation}

Next, we substitute the ground distance into the equation of Definition~\ref{def:wasserstein} and compute the Wasserstein distance using a network simplex method~\citep{peyre2019computational}. 

\textbf{Computational complexity.}
Naively, the computation of the Wasserstein Distance has a complexity of $\mathcal{O}(n^3log(n))$, with $n$ being the cardinality of the indexed set of node embeddings, i.e.,\ the number of nodes in the two graphs. Nevertheless, efficient speedup tricks can be employed. In particular, approximations relying on Sinkhorn regularisation have been proposed~\citep{cuturi2013sinkhorn}, some of which reduce the computational burden to \textit{near-linear time} while preserving accuracy~\citep{altschuler2017near}. Such speedup strategies become incredibly useful for larger data sets, i.e.,\ graphs with thousands of nodes, and can be easily integrated into our method. See  Appendix~\ref{app:runtime} for a practical discussion.


\begin{algorithm}[tb]
   \caption{Compute Wasserstein graph kernel}
   \label{alg:gwk}
\begin{algorithmic}
   \STATE {\bfseries Input:} Two graphs $G_1$, $G_2$; graph embedding scheme $f^H$; ground distance $d$; $\lambda$.
   \STATE {\bfseries Output:} kernel value $k_{WWL}(G_1,G_2)$.
   \\
   \STATE $X_{G_1} \leftarrow f^H(G_1); \; X_{G_2} \leftarrow f^H(G_2)$ \COMMENT{Generate node embeddings}
   \STATE $D \leftarrow \pairwisedist(X_{G_1},X_{G_2},d)$ \COMMENT{Compute the ground distance between each pair of nodes}
   \STATE $D_W(G_1,G_2) = \min_{P \in \Gamma} \left \langle P, D\right \rangle$ \COMMENT{Compute the Wasserstein distance}
   
   \STATE $k_W(G_1,G_2) \leftarrow e^{- \lambda D_W(G_1,G_2)}$
\end{algorithmic}
\end{algorithm}

\section{From Wasserstein distance to kernels}
\label{sec:kernels}

From the graph Wasserstein distance, one can construct a similarity measure to be used in a learning algorithm.
In this section, we propose a new graph kernel, state some claims about its (in)definiteness, and elaborate on how to use it for classifying graphs with continuous and categorical node labels.

\begin{definition}[Wasserstein Weisfeiler--Lehman]
\label{def:wwl}
Given a set of graphs $\mathcal{G} = \{G_1,\ldots, G_N\}$  and the GWD defined for each pair of graph on their WL embeddings, we define the Wasserstein Weisfeiler--Lehman~(WWL) kernel as 
\begin{equation}
    K_{\mathrm{\scriptscriptstyle{WWL}}} = e^{-\lambda D^{f_{\mathrm{\scriptscriptstyle{WL}}}}_W}.
\end{equation}
\end{definition}

This is an instance of a Laplacian kernel, which was shown to offer favourable conditions for positive definiteness in the case of non-Euclidean distances~\citep{feragen2015geodesic}. Obtaining the WWL kernel concludes the procedure described in Algorithm \ref{alg:gwk}. In the remainder of this section, we distinguish between the categorical WWL kernel, obtained on graphs with categorical labels, and the continuous WWL kernel, obtained on continuously attributed graphs via the graph embedding schemes described in Section~\ref{sec:NodeEmbeddings}.

For Euclidean spaces, obtaining positive definite kernels from distance functions is a well-studied topic~\citep{haasdonk2004learning}. 
However, the Wasserstein distance in its general form is not isometric, i.e.,\ there is no metric-preserving mapping to an $L^2$-norm, as the metric space it induces strongly depends on the chosen ground distance~\citep{figalli2011optimal}.
Therefore, despite being a metric, it is not necessarily possible to derive a positive definite kernel from the Wasserstein distance in its general formulation, because the classical approaches~\citep{haasdonk2004learning} cannot be applied here.
Nevertheless, as a consequence of using the Laplacian kernel~\citep{feragen2015geodesic}, we can show that, in the setting of categorical node labels, the obtained kernel is positive definite.
\begin{theorem}
The categorical WWL kernel is positive definite for all $\lambda > 0$.
\label{th:pd}
\end{theorem}

For a proof, see Sections~\ref{app:definiteness} and~\ref{app:discretewwl} in the Appendix.
By contrast, for the continuous case, establishing the definiteness of the obtained kernel remains an open problem. We refer the reader to Section~\ref{sec:Continuous embeddings} in the supplementary materials for
further discussions and conjectures.

Therefore, to ensure the theoretical and practical correctness of our results \emph{in the continuous case}, we employ recently developed methods for learning with indefinite kernels. Specifically, we use learning methods for Kre\u{\i}n spaces, which have been specifically designed to work with indefinite kernels~\citep{ong2004learning}; in general, kernels that are not positive definite induce reproducing kernel Kre\u{\i}n spaces~(RKKS). These spaces can be seen as a generalisation of reproducing kernel Hilbert spaces, with which they share similar mathematical properties, making them amenable to machine learning techniques.
Recent algorithms~\citep{loosli2015learning, oglic2018learning} are capable of solving learning problems in RKKS; their results indicate that there are clear benefits~(in terms of classification performance, for example) of learning in such spaces.
Therefore, when evaluating WWL, we will use a Kre\u{\i}n SVM~(KSVM, \citep{loosli2015learning}) as a classifier for the case of continuous attributes.

\begin{table}[t]
\centering
\caption{Classification accuracies on graphs with categorical node labels. Comparison of Weisfeiler--Lehman kernel (WL), optimal assignment kernel (WL-OA), and our method (WWL).}
\label{results_disc}

  %
  \renewrobustcmd{\bfseries}{\fontseries{b}\selectfont}
  \renewrobustcmd{\boldmath}{}
  \sisetup{%
    mode         = text,
    table-format = 2.2(2),
    tight-spacing,
  }
  \fontsize{9.5}{10.5}\selectfont
\begin{sc}
    \setlength{\tabcolsep}{5.00pt}
    
    \begin{tabular}{lcccccc}
    \toprule
    Method & MUTAG & PTC-MR & NCI1  & PROTEINS & D$\&$D & ENZYMES\\
    \midrule
      V & \num{85.39 \pm 0.73} & \num{58.35 \pm 0.20}\phantom{$^\ast$} & \num{64.22 \pm 0.11} & \num{72.12 \pm 0.19}\phantom{$^\ast$} & \num{78.24 \pm 0.28} & \num{22.72 \pm 0.56}\\
    E & \num{84.17 \pm 1.44} & \num{55.82 \pm 0.00}\phantom{$^\ast$} & \num{63.57 \pm 0.12} & \num{72.18 \pm 0.42}\phantom{$^\ast$} & \num{75.49 \pm 0.21} & \num{21.87 \pm 0.64}\\
    \midrule
    WL    & \num{85.78 \pm 0.83} & \num{61.21 \pm 2.28}\phantom{$^\ast$} & \num{85.83 \pm 0.09}  & \num{74.99 \pm 0.28}\phantom{$^\ast$} & \num{78.29 \pm 0.30} & \num{53.33 \pm 0.93}\\
    WL-OA & \bfseries\num{87.15 \pm 1.82} & \num{60.58 \pm 1.35}\phantom{$^\ast$} & \bfseries\num{86.08 \pm 0.27} &  \bfseries\num{76.37 \pm 0.30}$^\ast$ & \num{79.15 \pm 0.33} & \bfseries\num{58.97 \pm 0.82}\\
    \midrule
    WWL  & \bfseries\num{87.27 \pm 1.50} & \bfseries\num{66.31 \pm  1.21}$^\ast$ & \num{85.75 \pm 0.25} &  \num{74.28 \pm 0.56}\phantom{$^\ast$} & \bfseries\num{79.69 \pm 0.50} & \bfseries\num{59.13 \pm 0.80}\\
    \bottomrule
    \end{tabular}
\end{sc}
\end{table}

\begin{table}[t]
\centering
\caption{Classification accuracies on graphs with continuous node and/or edge attributes. Comparison of hash graph kernel (HGK-WL, HGK-SP), GraphHopper kernel (GH), and our method (WWL).}
\label{results_cont}

  %
  %
  \renewrobustcmd{\bfseries}{\fontseries{b}\selectfont}
  \renewrobustcmd{\boldmath}{}
  \sisetup{%
    mode         = text,
    table-format = 2.2(2),
    tight-spacing,
  }
\begin{sc}
 \setlength{\tabcolsep}{2.25pt}
 \fontsize{9}{10}\selectfont
    \begin{tabular}{lcccccccc}
        \toprule
        Method & ENZYMES           & PROTEINS         &  IMDB-B & BZR & COX2 & BZR-MD & COX2-MD\\
        \midrule
        VH-C & \num{47.15 \pm 0.79}\phantom{$^\ast$} & \num{60.79 \pm 0.12}\phantom{$^\ast$} & \num{71.64 \pm 0.49}\phantom{$^\ast$} & \num{74.82 \pm 2.13}\phantom{$^\ast$} & \num{48.51 \pm 0.63} & \num{66.58 \pm 0.97} & \num{64.89 \pm 1.06}  \\
        RBF-WL  & \num{68.43 \pm 1.47}\phantom{$^\ast$} & \num{75.43 \pm 0.28}\phantom{$^\ast$} & \num{72.06 \pm 0.34}\phantom{$^\ast$} & \num{80.96 \pm 1.67}\phantom{$^\ast$} &\num{75.45 \pm 1.53} &  \bfseries{\num{69.13 \pm 1.27}} & \num{71.83 \pm 1.61}\\
        \midrule
        HGK-WL  & \num{63.04 \pm 0.65}\phantom{$^\ast$} & \num{75.93 \pm 0.17}\phantom{$^\ast$} & \num{73.12 \pm 0.40}\phantom{$^\ast$} & \ \num{78.59 \pm 0.63}\phantom{$^\ast$} & \bfseries{\num{78.13 \pm 0.45}}  &   \bfseries{\num{68.94 \pm 0.65}} & \num{74.61 \pm 1.74} \\
        HGK-SP  & \num{66.36 \pm 0.37}\phantom{$^\ast$} & \num{75.78 \pm 0.17}\phantom{$^\ast$} & \num{73.06 \pm 0.27}\phantom{$^\ast$} & \num{76.42 \pm 0.72}\phantom{$^\ast$} & \num{72.57 \pm 1.18}  & \num{66.17 \pm 1.05} & \num{68.52 \pm 1.00} \\
        \midrule
        GH      & \num{65.65\pm 0.80}\phantom{$^\ast$}  & \num{74.78\pm 0.29}\phantom{$^\ast$}  & \num{72.35 \pm 0.55}\phantom{$^\ast$} & \num{76.49 \pm 0.99}\phantom{$^\ast$} & \num{76.41 \pm 1.39} &  \bfseries{\num{69.14 \pm 2.08}} &  \num{66.20 \pm 1.05} \\
        \midrule
        WWL     & \bfseries\num{73.25\pm 0.87}$^\ast$ & \bfseries\num{77.91 \pm  0.80}$^\ast$  & \bfseries\num{74.37 \pm 0.83}$^\ast$ &  \bfseries\num{84.42 \pm 2.03}$^\ast$ & \bfseries{\num{78.29 \pm 0.47}} & \bfseries{\num{69.76 \pm 0.94}} & \bfseries\num{76.33 \pm 1.02} \\ 
        \bottomrule
    \end{tabular}
\end{sc}
\end{table}

\section{Experimental evaluation}
\label{sec:results}

In this section, we analyse how the performance of WWL compares with state-of-the-art graph kernels. In particular, we empirically observe that WWL
\begin{inparaenum}[(1)]
    \item is competitive with the best graph kernel for categorically labelled data, and
    \item outperforms all the state-of-the-art graph kernels for attributed graphs. 
\end{inparaenum}

\subsection{Data sets}

We report results on real-world data sets from multiple sources~\citep{borgwardt2005protein, vishwanathan2010graph,shervashidze2011weisfeiler} and use either their continuous attributes or categorical labels for evaluation. In particular, \textsc{MUTAG}, \textsc{PTC-MR}, \textsc{NCI1}, and \textsc{D\&D} are equipped with categorical node labels only; \textsc{ENZYMES} and  \textsc{PROTEINS} have both categorical labels and continuous attributes; \textsc{IMDB-B}, \textsc{BZR}, and \textsc{COX2} only contain continuous attributes; finally, \textsc{BZR-MD} and \textsc{COX2-MD} have both continuous node attributes and edge weights. 
Further information on the data sets is available in Supplementary Table~\ref{table:data}. 
Additionally, we report results on synthetic data~(\textsc{Synthie} and \textsc{SYNTHETIC-new}) in Appendix~\ref{app:additional_ds}. All the data sets have been downloaded from \citet{repoker2016}.

\subsection{Experimental setup}
\label{sec:expsetup}
We compare WWL with state-of-the-art graph kernel methods from the literature and relevant baselines, which we trained ourselves on the same splits~(see below). In particular, for the categorical case, we compare with WL~\citep{shervashidze2009FastSK} and WL-OA~\citep{kriege2016oa} as well as with the vertex~(V) and edge~(E) histograms. Because~\citep{kriege2016oa} already showed that the WL-OA is superior to previous approaches, we do not include the whole set of kernels in our comparison. For the continuously attributed data sets, we compare with two instances of the hash graph kernel~(HGK-SP; HGK-WL)~\citep{morris2016faster} and with the GraphHopper~(GH)~\citep{feragen2013scalable}. For comparison, we additionally use a continuous vertex histogram~(VH-C), which is defined as a radial basis function (RBF) kernel between the sum of the graph node embeddings. Furthermore,
to highlight the benefits of using the Wasserstein distance in our method, we replace it with an RBF kernel. Specifically, given two graphs $G_1 = (V_1,E_1)$ and $G_2 = (V_2, E_2)$, with $|V_1|=n_1$ and $|V_2| = n_2$, we first compute the Gaussian kernel between each pair of the node embeddings obtained in the same fashion as for WWL; therefore, we obtain a kernel matrix between node embeddings $K' \in n_1 \times n_2$. Next, we sum up the values $K_{s}$ = $\sum_{i=1}^{n_1}\sum_{j=1}^{n_2}K'_{i,j}$ and set $K(G_1,G_2) = K_{s}$. This procedure is repeated for each pair of graphs to obtain the final graph kernel matrix. We refer to this baseline as RBF-WL.

As a classifier, we use an SVM~(or a KSVM in the case of WWL) and 10-fold cross-validation, selecting the parameters on the training set only. We repeat each cross-validation split 10 times and report the average accuracy. We employ the same split for each evaluated method, thereby guaranteeing a fully comparable setup among all evaluated methods. Please refer to Appendix~\ref{app:hypersel} for details on the hyperparameter selection.

\paragraph{Implementation and computing infrastructure}
Available Python implementations can be used to compute the WL kernel~\citep{sugiyama2018graphker} and the Wasserstein distance~\citep{flamary2017pot}. 
We leverage these resources and make our code publicly available\footnote{\href{https://github.com/BorgwardtLab/WWL}{https://github.com/BorgwardtLab/WWL}}. We use the original implementations provided by the respective authors to compute the WL-OA, HGK, and GH methods.
All our analyses were performed on a shared server running Ubuntu 14.04.5 LTS, with 4 CPUs (Intel Xeon E7-4860 v2 @ 2.60GHz) each with 12 cores and 24 threads, and 512 GB of RAM.

\subsection{Results and discussion}
\label{sec:resdisc}

The results are evaluated by classification accuracy and summarised in Table~\ref{results_disc} and Table~\ref{results_cont} for the categorical labels and continuous attributes, respectively\footnote{
The best performing methods up to the resolution implied by the standard deviation across repetitions are highlighted in boldface. Additionally, to evaluate significance we perform 2-sample $t$-tests with a significance threshold of $0.05$ and Bonferroni correction for multiple hypothesis testing within each data set, significantly outperforming methods are denoted by an asterisk.}.

\subsubsection{Categorical labels}

On the categorical data sets, WWL is comparable to the WL-OA kernel; however, it improves over the classical WL. In particular, WWL largely improves over WL-OA in \textsc{PTC-MR} and is slightly better on \textsc{D\&D}, whereas WL-OA is better on \textsc{NCI1} and \textsc{PROTEINS}. 

Unsurprisingly,  our approach is comparable to the WL-OA, whose main idea is to solve the optimal assignment problem by defining Dirac kernels on histograms of node labels, using multiple iterations of WL. 
This formulation is similar to the one we provide for categorical data, but it relies on the optimal assignment rather than the optimal transport; therefore, it requires one-to-one mappings instead of continuous transport maps. Besides, we solve the optimal transport problem on the concatenated embeddings, hereby jointly exploiting representations at multiple WL iterations. Contrarily, the WL-OA performs an optimal assignment at each iteration of WL and only combines them in the second stage.
However, the key advantage of WWL over WL-OA is its capacity to account for continuous attributes.

\subsubsection{Continuous attributes}

In this setting, WWL significantly outperforms the other methods on $4$ out of $7$ data sets, is better on another one, and is on a par on the remaining $2$. We further compute the average rank of each method in the continuous setting, with WWL scoring as first. The ranks calculated from Table~\ref{results_cont} are WWL = $1$, HGK-WL = $2.86$, RBF-WL = $3.29$, HGK-SP = $4.14$, and VH-C = $5.86$.  This is a remarkable improvement over the current state of the art, and it indeed establishes a new one. When looking at the average rank of the method, WWL always scores first. Therefore, we raise the bar in kernel graph classification for attributed graphs. As mentioned in Section~\ref{sec:kernels}, the kernel obtained from continuous attributes is not necessarily positive definite. However, we empirically observe the kernel matrices to be positive definite~(up to a numerical error), further supporting our theoretical considerations~(see Appendix~\ref{app:definiteness}).
In practice, the difference between the results obtained from classical SVMs in RKHS and the results obtained with the KSVM approach is negligible.

\paragraph{Comparison with hash graph kernels }
The hash graph kernel~(HGK) approach is somewhat related to our propagation scheme. By using multiple hashing functions, the HGK method is capable of extending certain existing graph kernels to the continuous setting. This helps to avoid the limitations of perfect hashing, which cannot express small differences in continuous attributes. A drawback of the random hashing performed by HGK is that it requires additional parameters and introduces a stochastic element to the kernel matrix computation. By contrast, our propagation scheme is fully continuous and uses the Wasserstein distance to capture small differences in distributions of continuous node attributes. Moreover, the observed performance gap suggests that an entirely continuous representation of the graphs provides clear benefits over the hashing.

\section{Conclusion}
\label{sec:conclusion}

In this paper, we present a new family of graph kernels, the Wasserstein Weisfeiler--Lehman~(WWL) graph kernels.
Our experiments show that WWL graph kernels outperform the state of the art for graph classification in the scenario of continuous node attributes, while matching the state of the art in the categorical setting.
As a line of research for future work, we see great potential in the runtime improvement, thus, allowing applications of our method on regimes and data sets with larger graphs. In fact, preliminary experiments~(see Section~\ref{app:runtime} as well as Figure~\ref{fig:runtime} in the Appendix) already confirm the benefit of Sinkhorn regularisation when the average number of nodes in the graph increases. In parallel, it would be beneficial to derive approximations of the explicit feature representations in the RKKS, as this would also provide a consistent speedup. We further envision that major theoretical contributions could be made by defining theoretical bounds to ensure the positive definiteness of the WWL kernel in the case of continuous node attributes.
Finally, optimisation objectives based on optimal transport could be employed to develop new algorithms based on graph neural networks~\citep{duvenaud2015convolutional, kipf2017semi}. On a more general level, our proposed method provides a solid foundation of the use of optimal transport theory for kernel methods and highlights the large potential of optimal transport for machine learning.

\subsubsection*{Acknowledgments}
\label{sec:acknowledgments}

This work was funded in part by the Horizon 2020 project CDS-QUAMRI, Grant No. 634541 (E.G., K.B.), the Alfried Krupp Prize for Young University Teachers of the Alfried Krupp von Bohlen und Halbach-Stiftung (B.R., K.B.), and the SNSF Starting Grant ``Significant Pattern
Mining'' (F.L., K.B.).

\bibliographystyle{abbrvnat}
\bibliography{paper}

\begin{thebibliography}{47}
\providecommand{\natexlab}[1]{#1}
\providecommand{\url}[1]{\texttt{#1}}
\expandafter\ifx\csname urlstyle\endcsname\relax
  \providecommand{\doi}[1]{doi: #1}\else
  \providecommand{\doi}{doi: \begingroup \urlstyle{rm}\Url}\fi

\bibitem[Altschuler et~al.(2017)Altschuler, Weed, and
  Rigollet]{altschuler2017near}
J.~Altschuler, J.~Weed, and P.~Rigollet.
\newblock Near-linear time approximation algorithms for optimal transport via
  sinkhorn iteration.
\newblock In \emph{Advances in Neural Information Processing Systems~30}, pages
  1964--1974, 2017.

\bibitem[Arjovsky et~al.(2017)Arjovsky, Chintala, and
  Bottou]{arjovsky2017wasserstein}
M.~Arjovsky, S.~Chintala, and L.~Bottou.
\newblock Wasserstein {GAN}.
\newblock \emph{arXiv preprint arXiv:1701.07875}, 2017.

\bibitem[Balcan et~al.(2008)Balcan, Blum, and Srebro]{balcan2008theory}
M.-F. Balcan, A.~Blum, and N.~Srebro.
\newblock A theory of learning with similarity functions.
\newblock \emph{Machine Learning}, 72\penalty0 (1-2):\penalty0 89--112, 2008.

\bibitem[Berg et~al.(1984)Berg, Christensen, and Ressel]{berg1984harmonic}
C.~Berg, J.~P.~R. Christensen, and P.~Ressel.
\newblock \emph{Harmonic analysis on semigroups}.
\newblock Springer, Heidelberg, Germany, 1984.

\bibitem[Borgwardt and Kriegel(2005)]{borgwardt2005shortest}
K.~M. Borgwardt and H.-P. Kriegel.
\newblock Shortest-path kernels on graphs.
\newblock In \emph{Proceedings of the Fifth IEEE International Conference on
  Data Mining}, pages 74--81, 2005.

\bibitem[Borgwardt et~al.(2005)Borgwardt, Ong, Sch{\"o}nauer, Vishwanathan,
  Smola, and Kriegel]{borgwardt2005protein}
K.~M. Borgwardt, C.~S. Ong, S.~Sch{\"o}nauer, S.~Vishwanathan, A.~J. Smola, and
  H.-P. Kriegel.
\newblock Protein function prediction via graph kernels.
\newblock \emph{Bioinformatics}, 21:\penalty0 i47--i56, 2005.

\bibitem[Bridson and H{\"a}fliger(2013)]{bridson2013metric}
M.~R. Bridson and A.~H{\"a}fliger.
\newblock \emph{Metric spaces of non-positive curvature}.
\newblock Springer, Heidelberg, Germany, 2013.

\bibitem[Cuturi(2013)]{cuturi2013sinkhorn}
M.~Cuturi.
\newblock Sinkhorn distances: Lightspeed computation of optimal transport.
\newblock In \emph{Advances in Neural Information Processing Systems~26}, pages
  2292--2300, 2013.

\bibitem[Duvenaud et~al.(2015)Duvenaud, Maclaurin, Iparraguirre, Bombarell,
  Hirzel, Aspuru-Guzik, and Adams]{duvenaud2015convolutional}
D.~K. Duvenaud, D.~Maclaurin, J.~Iparraguirre, R.~Bombarell, T.~Hirzel,
  A.~Aspuru-Guzik, and R.~P. Adams.
\newblock Convolutional networks on graphs for learning molecular fingerprints.
\newblock In \emph{Advances in Neural Information Processing Systems~28}, pages
  2224--2232, 2015.

\bibitem[Feragen et~al.(2013)Feragen, Kasenburg, Petersen, de~Bruijne, and
  Borgwardt]{feragen2013scalable}
A.~Feragen, N.~Kasenburg, J.~Petersen, M.~de~Bruijne, and K.~Borgwardt.
\newblock Scalable kernels for graphs with continuous attributes.
\newblock In \emph{Advances in Neural Information Processing Systems~26}, pages
  216--224, 2013.

\bibitem[Feragen et~al.(2015)Feragen, Lauze, and Hauberg]{feragen2015geodesic}
A.~Feragen, F.~Lauze, and S.~Hauberg.
\newblock Geodesic exponential kernels: When curvature and linearity conflict.
\newblock In \emph{Proceedings of the IEEE Conference on Computer Vision and
  Pattern Recognition}, pages 3032--3042, 2015.

\bibitem[Figalli and Villani(2011)]{figalli2011optimal}
A.~Figalli and C.~Villani.
\newblock Optimal transport and curvature.
\newblock In \emph{Nonlinear PDE’s and Applications}, pages 171--217.
  Springer, Heidelberg, Germany, 2011.

\bibitem[Flamary and Courty(2017)]{flamary2017pot}
R.~Flamary and N.~Courty.
\newblock {POT}: {P}ython {O}ptimal {T}ransport library, 2017.
\newblock URL \url{https://github.com/rflamary/POT}.

\bibitem[Frogner et~al.(2015)Frogner, Zhang, Mobahi, Araya, and
  Poggio]{frogner2015learning}
C.~Frogner, C.~Zhang, H.~Mobahi, M.~Araya, and T.~A. Poggio.
\newblock Learning with a {W}asserstein loss.
\newblock In \emph{Advances in Neural Information Processing Systems~28}, pages
  2053--2061, 2015.

\bibitem[Fr\"{o}hlich et~al.(2005)Fr\"{o}hlich, Wegner, Sieker, and
  Zell]{frohlich2005}
H.~Fr\"{o}hlich, J.~K. Wegner, F.~Sieker, and A.~Zell.
\newblock Optimal assignment kernels for attributed molecular graphs.
\newblock In \emph{Proceedings of the 22nd International Conference on Machine
  Learning}, pages 225--232, 2005.

\bibitem[Gardner et~al.(2017)Gardner, Duncan, Kanno, and
  Selmic]{gardner2017definiteness}
A.~Gardner, C.~A. Duncan, J.~Kanno, and R.~R. Selmic.
\newblock On the definiteness of {E}arth {M}over's {D}istance and its relation
  to set intersection.
\newblock \emph{IEEE Transactions on Cybernetics}, 2017.

\bibitem[Haasdonk and Bahlmann(2004)]{haasdonk2004learning}
B.~Haasdonk and C.~Bahlmann.
\newblock Learning with distance substitution kernels.
\newblock In \emph{DAGM-Symposium}, 2004.

\bibitem[Haussler(1999)]{haussler1999convolution}
D.~Haussler.
\newblock Convolution kernels on discrete structures.
\newblock Technical report, Department of Computer Science, University of
  California, 1999.

\bibitem[Kashima et~al.(2003)Kashima, Tsuda, and Inokuchi]{kashima2003}
H.~Kashima, K.~Tsuda, and A.~Inokuchi.
\newblock Marginalized kernels between labeled graphs.
\newblock In \emph{Proceedings of the 20th International Conference on Machine
  Learning}, pages 321--328, 2003.

\bibitem[Kersting et~al.(2016)Kersting, Kriege, Morris, Mutzel, and
  Neumann]{repoker2016}
K.~Kersting, N.~M. Kriege, C.~Morris, P.~Mutzel, and M.~Neumann.
\newblock Benchmark data sets for graph kernels, 2016.
\newblock URL \url{http://graphkernels.cs.tu-dortmund.de}.

\bibitem[Kipf and Welling(2017)]{kipf2017semi}
T.~N. Kipf and M.~Welling.
\newblock Semi-supervised classification with graph convolutional networks.
\newblock In \emph{5th International Conference on Learning Representations},
  2017.

\bibitem[Klicpera et~al.(2019)Klicpera, Bojchevski, and
  Günnemann]{klicpera2018combining}
J.~Klicpera, A.~Bojchevski, and S.~Günnemann.
\newblock Combining neural networks with personalized pagerank for
  classification on graphs.
\newblock In \emph{7th International Conference on Learning Representations},
  2019.

\bibitem[Kolouri et~al.(2016)Kolouri, Zou, and Rohde]{kolouri2016sliced}
S.~Kolouri, Y.~Zou, and G.~K. Rohde.
\newblock Sliced {W}asserstein kernels for probability distributions.
\newblock In \emph{Proceedings of the IEEE Conference on Computer Vision and
  Pattern Recognition}, pages 5258--5267, 2016.

\bibitem[Kriege and Mutzel(2012)]{kriege2012subgraph}
N.~Kriege and P.~Mutzel.
\newblock Subgraph matching kernels for attributed graphs.
\newblock In \emph{Proceedings of the 29th International Conference on Machine
  Learning}, pages 1015--1022, 2012.

\bibitem[Kriege et~al.(2016)Kriege, Giscard, and Wilson]{kriege2016oa}
N.~M. Kriege, P.-L. Giscard, and R.~C. Wilson.
\newblock On valid optimal assignment kernels and applications to graph
  classification.
\newblock In \emph{Advances in Neural Information Processing Systems~29}, pages
  1623--1631, 2016.

\bibitem[Loosli et~al.(2015)Loosli, Canu, and Ong]{loosli2015learning}
G.~Loosli, S.~Canu, and C.~S. Ong.
\newblock Learning {SVM} in {K}re{\u\i}n spaces.
\newblock \emph{IEEE Transactions on Pattern Analysis and Machine
  Intelligence}, 38\penalty0 (6):\penalty0 1204--1216, 2015.

\bibitem[Morris et~al.(2016)Morris, Kriege, Kersting, and
  Mutzel]{morris2016faster}
C.~Morris, N.~M. Kriege, K.~Kersting, and P.~Mutzel.
\newblock Faster kernels for graphs with continuous attributes via hashing.
\newblock In \emph{Proceedings of the 16th IEEE International Conference on
  Data Mining}, pages 1095--1100, 2016.

\bibitem[Neumann et~al.(2016)Neumann, Garnett, Bauckhage, and
  Kersting]{neumann2016}
M.~Neumann, R.~Garnett, C.~Bauckhage, and K.~Kersting.
\newblock Propagation kernels: efficient graph kernels from propagated
  information.
\newblock \emph{Machine Learning}, 102\penalty0 (2):\penalty0 209--245, 2016.

\bibitem[Oglic and G{\"a}rtner(2018)]{oglic2018learning}
D.~Oglic and T.~G{\"a}rtner.
\newblock Learning in reproducing kernel kre{\i}n spaces.
\newblock In \emph{Proceedings of the 35th International Conference on Machine
  Learning}, pages 3859--3867, 2018.

\bibitem[Ong et~al.(2004)Ong, Mary, Canu, and Smola]{ong2004learning}
C.~S. Ong, X.~Mary, S.~Canu, and A.~J. Smola.
\newblock Learning with non-positive kernels.
\newblock In \emph{Proceedings of the 21st International Conference on Machine
  Learning}, 2004.

\bibitem[Peyr{\'e} et~al.(2019)Peyr{\'e}, Cuturi,
  et~al.]{peyre2019computational}
G.~Peyr{\'e}, M.~Cuturi, et~al.
\newblock Computational optimal transport.
\newblock \emph{Foundations and Trends{\textregistered} in Machine Learning},
  11\penalty0 (5-6):\penalty0 355--607, 2019.

\bibitem[Rabin et~al.(2011)Rabin, Peyr{\'e}, Delon, and
  Bernot]{rabin2011wasserstein}
J.~Rabin, G.~Peyr{\'e}, J.~Delon, and M.~Bernot.
\newblock Wasserstein barycenter and its application to texture mixing.
\newblock In \emph{International Conference on Scale Space and Variational
  Methods in Computer Vision}, pages 435--446, 2011.

\bibitem[Rieck et~al.(2019)Rieck, Bock, and Borgwardt]{Rieck19b}
B.~Rieck, C.~Bock, and K.~Borgwardt.
\newblock A persistent {W}eisfeiler--{L}ehman procedure for graph
  classification.
\newblock In \emph{Proceedings of the 36th International Conference on Machine
  Learning}, pages 5448--5458, 2019.

\bibitem[Rubner et~al.(2000)Rubner, Tomasi, and Guibas]{rubner2000earth}
Y.~Rubner, C.~Tomasi, and L.~J. Guibas.
\newblock The {E}arth {M}over's {D}istance as a metric for image retrieval.
\newblock \emph{International Journal of Computer Vision}, 40\penalty0
  (2):\penalty0 99--121, 2000.

\bibitem[Sch{\"o}lkopf(2001)]{scholkopf2001kernel}
B.~Sch{\"o}lkopf.
\newblock The kernel trick for distances.
\newblock In \emph{Advances in Neural Information Processing Systems~13}, pages
  301--307, 2001.

\bibitem[Sch{\"o}lkopf and Smola(2002)]{scholkopf2002learning}
B.~Sch{\"o}lkopf and A.~J. Smola.
\newblock \emph{Learning with kernels: support vector machines, regularization,
  optimization, and beyond}.
\newblock MIT press, 2002.

\bibitem[Shervashidze and Borgwardt(2009)]{shervashidze2009FastSK}
N.~Shervashidze and K.~M. Borgwardt.
\newblock Fast subtree kernels on graphs.
\newblock In \emph{Advances in Neural Information Processing Systems~22}, pages
  1660--1668, 2009.

\bibitem[Shervashidze et~al.(2011)Shervashidze, Schweitzer, Leeuwen, Mehlhorn,
  and Borgwardt]{shervashidze2011weisfeiler}
N.~Shervashidze, P.~Schweitzer, E.~J.~v. Leeuwen, K.~Mehlhorn, and K.~M.
  Borgwardt.
\newblock {W}eisfeiler-{L}ehman graph kernels.
\newblock \emph{Journal of Machine Learning Research}, 12:\penalty0 2539--2561,
  2011.

\bibitem[Shin-Ichi(2012)]{Shin-Ichi2012}
O.~Shin-Ichi.
\newblock Barycenters in {A}lexandrov spaces of curvature bounded below.
\newblock \emph{Advances in Geometry}, 14\penalty0 (4):\penalty0 571--587,
  2012.

\bibitem[Stevens(1946)]{stevens1946theory}
S.~S. Stevens.
\newblock On the theory of scales of measurement.
\newblock \emph{Science}, 103\penalty0 (2684):\penalty0 677--680, 1946.

\bibitem[Sugiyama et~al.(2018)Sugiyama, Ghisu, Llinares-López, and
  Borgwardt]{sugiyama2018graphker}
M.~Sugiyama, M.~E. Ghisu, F.~Llinares-López, and K.~Borgwardt.
\newblock graphkernels: R and python packages for graph comparison.
\newblock \emph{Bioinformatics}, 34\penalty0 (3):\penalty0 530--532, 2018.

\bibitem[Turner et~al.(2014)Turner, Mileyko, Mukherjee, and Harer]{Turner2014}
K.~Turner, Y.~Mileyko, S.~Mukherjee, and J.~Harer.
\newblock Fr{\'e}chet means for distributions of persistence diagrams.
\newblock \emph{Discrete {\&} Computational Geometry}, 52:\penalty0 44--70,
  2014.

\bibitem[Vert(2008)]{vert2008optimal}
J.-P. Vert.
\newblock The optimal assignment kernel is not positive definite.
\newblock \emph{arXiv preprint arXiv:0801.4061}, 2008.

\bibitem[Villani(2008)]{villani2008optimal}
C.~Villani.
\newblock \emph{Optimal transport: old and new}, volume 338.
\newblock Springer, Heidelberg, Germany, 2008.

\bibitem[Vishwanathan et~al.(2010)Vishwanathan, Schraudolph, Kondor, and
  Borgwardt]{vishwanathan2010graph}
S.~V.~N. Vishwanathan, N.~N. Schraudolph, R.~Kondor, and K.~M. Borgwardt.
\newblock Graph kernels.
\newblock \emph{Journal of Machine Learning Research}, 11:\penalty0 1201--1242,
  2010.

\bibitem[Xu et~al.(2019)Xu, Luo, Zha, and Duke]{xu2019gromov}
H.~Xu, D.~Luo, H.~Zha, and L.~C. Duke.
\newblock Gromov--{W}asserstein learning for graph matching and node embedding.
\newblock In \emph{Proceedings of the 36th International Conference on Machine
  Learning}, pages 6932--6941, 2019.

\bibitem[Yanardag and Vishwanathan(2015)]{yanardag2015deep}
P.~Yanardag and S.~Vishwanathan.
\newblock Deep graph kernels.
\newblock In \emph{Proceedings of the 21th ACM SIGKDD International Conference
  on Knowledge Discovery and Data Mining}, pages 1365--1374, 2015.

\end{thebibliography}

\clearpage
\appendix

\counterwithin{figure}{section}
\counterwithin{table}{section}

\section{Appendix}\label{sec:Appendix}

%

\subsection{Extended considerations on WWL definiteness}
\label{app:definiteness}

We will now discuss the positive definite nature of our WWL kernel. 

In general, whether distances obtained from optimal transport problems can be used to create positive definite kernels remains an open research question. Several attempts to draw general conclusions on the definiteness of the Wasserstein distance were unsuccessful, but insightful results on particular cases were obtained along the way. 
First, we collect some of these contributions and use them to prove that our WWL kernel for categorical embeddings is positive definite. 
Next, we elaborate further on the continuous embeddings case, for which we provide conjectures on practical conditions to obtain a positive definite kernel.

Before proceeding, let us reiterate some useful notions.
\begin{definition} \cite{scholkopf2002learning}
A symmetric function $k\colon \mathcal{X} \times \mathcal{X} \to \mathbb{R}$ is called a \textup{positive definite (pd) kernel} if it satisfies the condition
\begin{equation}
\label{eq:pdkernel}
\sum_{i,j=1}^{n}c_ic_jK_{ij} \geq 0, \, \, \textit{with} \, \, K_{ij} = k(x_i,x_j),
\end{equation}
for every $c_i\in\mathbb{R}$, $n\in\mathbb{N}$ and $x_{i}\in \mathcal{X}$.
\end{definition}

The matrix of kernel values $K$ with entries $K_{ij}$ is called the \emph{Gram matrix} of $k$ with respect to $x_1,\ldots,x_n$. A \textit{conditional} positive definite~(cpd) kernel is a function that satisfies Equation \ref{eq:pdkernel} for all $c_i \in \mathbb{R}$ with $\sum_{i=1}^n c_i = 0$. 
By analogy, a conditional negative definite~(cnd) kernel is a function that satisfies $\sum_{i,j=1}^{n}c_ic_jK_{ij} \leq 0$ for all $c_i \in \mathbb{R}$ with $\sum_{i=1}^n c_i = 0$.

For Euclidean spaces, obtaining kernels from distance functions is a well-studied topic.

\begin{proposition}
\citep{haasdonk2004learning}
\label{prop:haasdonk}
Let $d(x,x')$ be a symmetric, non-negative distance function with $d(x,x)=0$. If $d$ is isometric to an $L^2$-norm, then 
\begin{equation}
  k_d^{\mathrm{nd}}(x,x') = -d(x,x')^\beta, \;\;\beta \in \left [ 0, 2 \right ]
\end{equation}
is a valid cpd kernel.
\end{proposition}

However, the Wasserstein distance in its general form is not isometric to an $L^2$-norm, as the metric space it induces strongly depends on the chosen ground distance \citep{figalli2011optimal}.
Recently, \citet{feragen2015geodesic} argued that many types of data, including probability distributions, do not always reside in Euclidean spaces. Therefore, they define the family of exponential kernels relying on a non-Euclidean distance $d$ as 
\begin{equation}
\label{eq:geodesicexpkernels}
    k(x,x') = e^{-\lambda d(x,x')^q} \quad \text{for} \quad \lambda, q > 0,
\end{equation}
and, based on earlier considerations from \citet{berg1984harmonic}, show that, under certain conditions, the Laplacian kernel ($q=1$ in Equation~\ref{eq:geodesicexpkernels}) is positive definite.

\begin{proposition}
\citep{feragen2015geodesic}
\label{prop:feragen2}
The geodesic Laplacian kernel is positive definite for all $\lambda > 0$ if and only if the geodesic distance $d$ is conditional negative definite.
\end{proposition}

Once again, considerations on the negative definiteness of Wasserstein distance functions cannot be made on the general level. Certain ground distances, however, \emph{guarantee} the negative definiteness of the resulting Wasserstein distance. In particular, the Wasserstein distance with the discrete metric (i.e.,\ $\rho$ in Equation \ref{eq:hammingdist}) as the ground distance was proved to be conditional negative definite \citep{gardner2017definiteness}.

We will now leverage these results to prove that the Wasserstein distance equipped with the Hamming ground distance is conditional negative definite; therefore, it yields positive definite kernels for the categorical WL embeddings.

\subsubsection{The case of categorical embeddings}
\label{app:discretewwl}
When generating node embeddings using the Weisfeiler--Lehman labelling scheme with a shared dictionary across all the graphs, the solutions to the optimal transport problem are also shared across iterations.
We denote the Weisfeiler--Lehman embedding scheme as defined in Definition \ref{def:dwl_embedding} as $f^H_{\mathrm{\scriptscriptstyle{WL}}}$, and let $D^{f_{\mathrm{\scriptscriptstyle{WL}}}}_W$ be the corresponding GWD on a set of graphs $\mathcal{G}$ with categorical labels. Let $d_{\mathrm{Ham}}(v,v')$ of Equation~\ref{eq:hammingdist} be the ground distance of $D^{f_{\mathrm{\scriptscriptstyle{WL}}}}_W$. Then, the following useful results hold. 

\begin{lemma}
\label{lemma:discrete_distance}
If a transportation plan $\gamma$ with transport matrix $P$ is optimal in the sense of Definition~\ref{def:wasserstein} for distances $d_{\mathrm{Ham}}$ between embeddings obtained with $f^{H}_{\scriptscriptstyle{WL}}$, then it is also optimal for the discrete distances $d_{\mathrm{disc}}$ between the $H$-th iteration values obtained with the Weisfeiler--Lehman procedure.
\end{lemma}
\textit{Proof.} See Appendix \ref{proof:discrete_distance}.

\begin{lemma}
\label{lemma:besttransplan}
  If a transportation plan $\gamma$ with transport matrix $P$ is optimal in the sense of Definition~\ref{def:wasserstein} for distances $d_{\mathrm{Ham}}$ between embeddings obtained with $f^{H}_{\mathrm{\scriptscriptstyle{WL}}}$, then it is also optimal for distances $d_{\mathrm{Ham}}$ between embeddings obtained with $f^{H-1}_{\mathrm{\scriptscriptstyle{WL}}}$.
\end{lemma}
\textit{Proof.} See Appendix \ref{proof:besttransplan}.

Therefore, we postulate that the Wasserstein distance between categorical WL node embeddings is a conditional negative definite function.
\begin{theorem} 
\label{th:neg}
$D^{f_{\mathrm{\scriptscriptstyle{WL}}}}_W(\cdot,\cdot)$ is a conditional negative definite function.
\end{theorem}
\textit{Proof.} See Appendix \ref{proof:theorem}.

\textbf{Proof of Theorem \ref{th:pd}.}
Theorem \ref{th:neg} in light of Proposition~\ref{prop:feragen2} implies that the WWL kernel of Definition~\ref{def:wwl} is positive definite for all $\lambda > 0$. \hfill $\square$

We will now consider the case of the definiteness of kernels in the continuous setting. 

\subsubsection{The case of continuous embeddings}\label{sec:Continuous embeddings}

On one hand, in the categorical case, we proved the positive definiteness of our kernel. On the other hand, the continuous case is considerably harder to tackle.
We conjecture that, under certain conditions, the same might hold for continuous features. 
Although we do not have a formal proof yet, in what follows, we discuss arguments to support this conjecture, which seems to agree with our empirical findings.\footnote{We observe that for all considered data sets, after standardisation of the input features before the embedding scheme, GWD matrices are conditional negative definite.}

The \emph{curvature} of the metric space induced by the Wasserstein metric for a given ground distance plays an important role. First, we need to define \emph{Alexandrov spaces}.
\begin{definition}[Alexandrov space]
    Given a metric space and a real number $k$, the space is called an Alexandrov space if its sectional curvature is $\geq k$.
\end{definition}
Roughly speaking, the curvature indicates to what extent a geodesic triangle will be deformed in the space. The case of $k = 0$ is special as no distortion is happening here---hence, spaces that satisfy this property are called \emph{flat}. The concept of Alexandrov spaces is required in the following proposition, taken from a theorem by \citet{feragen2015geodesic}, which shows the relationship between a kernel and its underlying metric space.
\begin{proposition}
\label{prop:feragen}
The geodesic Gaussian kernel (i.e., $q=2$ in Equation \ref{eq:geodesicexpkernels}) is positive definite for all $\lambda > 0$ if and only if the underlying metric space $(X, d)$ is flat in the sense of Alexandrov, i.e.,\ if any geodesic triangle in $X$ can be isometrically embedded in a Euclidean space.
\end{proposition}
However, it is unlikely that the space induced by the Wasserstein distance is locally flat, as not even the geodesics~(i.e.,\ a generalisation of the shortest path to arbitrary metric spaces) between graph embeddings are necessarily unique, as we subsequently show. Hence,
we use the \emph{geodesic Laplacian kernel} instead of the Gaussian one because it poses less strict requirements on the induced space, as stated in Proposition~\ref{prop:feragen2}. Specifically, the metric used in the kernel function needs to be cnd. We cannot directly prove this yet, but we can prove that the converse is not true. To this end, we first notice that the metric space induced by the GWD, which we refer to as $X$, does \emph{not} have a curvature that is bounded from above.

\begin{definition}
A metric space $(X,d)$ is said to be $\CAT(k)$ if its curvature is bounded by some real number $k>0$ from above. This can also be seen as a ``relaxed'' definition, or generalisation, of a Riemannian manifold.
\end{definition}

\begin{theorem}
$X$ is not in $\CAT(k)$ for any $k > 0$, meaning that its curvature is \emph{not} bounded by any $k > 0$ from above.
\end{theorem}
\begin{proof}
  This follows from a similar argument presented by \citet{Turner2014}. We briefly sketch the argument.
  Let $G$ and $G'$ be two graphs.
  Assume that $X$ is a $\CAT(k)$ space for some $k > 0$.
  Then, it follows~\citep[Proposition 2.11, p.\ 23]{bridson2013metric} that if 
  $D^{f_{\mathrm{WL}}}_{W}(G,G') < \pi^2 / k$, there is a \emph{unique} geodesic between them. However, we can construct a family of graph embeddings for which this is not the case. To this end, let $\epsilon > 0$ and $f_\mathrm{WL}(G)$ and $f_\mathrm{WL}(G')$ be two graph embeddings with node embeddings $a_1 = (0, 0)$, $ a_2 = ( \epsilon, \epsilon)$ as well as $b_1 = (0,\epsilon)$ and $b_2 = (\epsilon,0)$, respectively. Because we use the Euclidean distance as a ground distance, there will be two optimal transport plans: the first maps $a_1$ to $b_1$ and $a_2$ to $b_2$, whereas the second maps $a_1$ to $b_2$ and $a_2$ to $b_1$. Hence, we have found two geodesics that connect $G$ and $G'$. Because we may choose $\epsilon$ to be arbitrarily small, the space cannot be $\CAT(k)$ for $k > 0$.
\end{proof}
Although this does not provide an upper bound on the curvature, we have the following conjecture.
\begin{conjecture}
$X$ is an Alexandrov space with curvature bounded from below by zero.
\end{conjecture}
For a proof idea, we refer to \citet{Turner2014}; the main argument involves characterizing the distance between triples of graph embeddings. This conjecture is helpful insofar as being a nonnegatively curved Alexandrov space is a necessary prerequisite for $X$ to be a Hilbert space~\citep{Shin-Ichi2012}. In turn, \citet{feragen2015geodesic} shows that cnd metrics and Hilbert spaces are intricately linked. 
Thus, we have some hope in obtaining a cnd metric, even though we currently lack a proof.
Our empirical results, however, indicate that it is possible to turn the GWD into a cnd metric with proper normalisation. 
Intuitively, for high-dimensional input spaces, standardisation of input features changes the curvature of the induced space by making it locally (almost) flat.

To support this argumentation, we refer to an existing way to ensure positive definiteness. One can use an alternative to the classical Wasserstein distance denoted as the sliced Wasserstein \citep{rabin2011wasserstein}. The idea is to project high-dimensional distributions into one-dimensional spaces, hereby calculating the Wasserstein distance as a combination of one-dimensional representations. 
\citet{kolouri2016sliced} showed that each of the one-dimensional Wasserstein distances is conditional negative definite.
The kernel on high-dimensional representations is then defined as a combination of the one-dimensional positive definite counterparts.

\subsection{Proof of Lemma \ref{lemma:discrete_distance}}
\label{proof:discrete_distance}
\begin{proof}
We recall the matrix notation introduced in Equation \ref{eq:wassdiscreteformulation} of the main paper, where $M$ is the cost or distance matrix, $P \in \Gamma$ is a transport matrix (or joint probability), and $\langle \cdot, \cdot \rangle$ is the Frobenius dot product.
Because we give equal weight (i.e.,\ equal probability mass) to each of the vectors in each set, $\Gamma$ contains all nonnegative $n \times n'$ matrices $P$ with
$$
\sum_{i=1}^{n} p_{ij} = \frac{1}{n'} \;\;,\;\;
\sum_{j=1}^{n'} p_{ij} = \frac{1}{n} \;\;,\;\;
p_{ij} \geq 0 \;\; \forall i, j
$$
  For notation simplicity, let us denote by $D_{\mathrm{Ham}}^h$ the Hamming matrix $D_{\mathrm{Ham}}(f^h_{\mathrm{\scriptscriptstyle{WL}}}(G), f^h_{\mathrm{\scriptscriptstyle{WL}}}(G'))$, where the $ij$-th entry is given by the Hamming distance between the embedding of the $i$-th node of graph $G$ and the embedding of the $j$-th node of graph $G'$ at iteration $h$. 
Similarly, we define $D_{\mathrm{disc}}^h$ to be the discrete metric distance matrix, where the $ij$-th entry is given by the discrete distance between feature $h$ of node embedding $i$ of graph $G$ and feature $h$ of node embedding $j$ of graph $G'$.
  It is easy to see that $[D_{\mathrm{Ham}}^h]_{ij} \in [0,1]$ and $[D_{\mathrm{disc}}^h]_{ij} \in \{0,1\}$ and, by definition,
$$
D_{\mathrm{Ham}}^H = \frac{1}{H} \sum_{h=0}^H D_{\mathrm{disc}}^h.
$$

Moreover, because of the WL procedure, two labels that are different at iteration $h$ will also be different at iteration $h+1$. Hence, the following identity holds:
$$
\left [D_{\mathrm{Ham}}^h \right ]_{ij} \leq
\left [D_{\mathrm{disc}}^{h} \right ]_{ij},
$$
which implies that $[D_{\mathrm{Ham}}^h]_{ij} =0 \iff [D_{\mathrm{disc}}^h]_{ij} =0$. 
An optimal transportation plan $P^{h}$ for $f^{h}_{\mathrm{\scriptscriptstyle{WL}}}$ embeddings satisfies
$$
\left \langle P^{h}, D_{\mathrm{Ham}}^{h}\right \rangle \leq \left \langle P, D_{\mathrm{Ham}}^{h}\right \rangle \; \forall P \in \Gamma.
$$
Assuming that $P^h$ is not optimal for $D_d^h$, we can define $P^*$ such that
$$
\left \langle P^{*}, D_{\mathrm{disc}}^{h}\right \rangle < \left \langle P^h, D_{\mathrm{disc}}^{h}\right \rangle.
$$
Because the entries of $D_{\mathrm{disc}}^h$ are either $0$ or $1$, we can define the set of indices tuples $\mathcal{H} = \left \{(i,j) \; | \; [D_{\mathrm{disc}}^h]_{ij} =1 \right \}$ and rewrite the inequality as
$$
\sum_{i,j \in \mathcal{H}} p^*_{ij} < \sum_{i,j \in \mathcal{H}} p^h_{ij}.
$$
Considering the constraints on the entries of $P^*$ and $P^h$, namely $\sum_{i,j} p^*_{ij} = \sum_{i,j} p^h_{ij} = 1$, this implies that, by rearranging the transport map, there is more mass that could be transported at $0$ cost. In our formalism,
$$
\sum_{i,j \notin \mathcal{H}} p^*_{ij} > \sum_{i,j \notin \mathcal{H}} p^h_{ij}.
$$
However, as stated before, entries of $D_d^h$ that are $0$ are also $0$ in $D_\mathrm{Ham}^h$. Therefore, a better transport plan $P^*$ would also be optimal for $D_\mathrm{Ham}^h$:
$$
\left \langle P^{*}, D_{\mathrm{Ham}}^{h}\right \rangle < \left \langle P^h, D_{\mathrm{Ham}}^{h}\right \rangle,
$$
which contradicts the optimality assumption above. Hence, $P^h$ is also optimal for $D_\mathrm{disc}^H$.

\end{proof}

\subsection{Proof of Lemma \ref{lemma:besttransplan}}
\label{proof:besttransplan}
\begin{proof}
Intuitively, the transportation plan at iteration $h$ is a ``refinement'' of the transportation plan at iteration $h-1$, where only a subset of the optimal transportation plans remains optimal for the new cost matrix $D_H^h$.
Using the same notation as for the Proof in Appendix \ref{proof:discrete_distance},
and considering the WL procedure, two labels that are different at iteration $h$ will also be different at iteration $h+1$. Hence, the following identities hold:
$$
\left [D_{\mathrm{Ham}}^h \right ]_{ij} \leq
\left [D_{\mathrm{Ham}}^{h+1} \right ]_{ij}\;\; 
\left [D_{\mathrm{disc}}^h \right ]_{ij} \leq
\left [D_{\mathrm{disc}}^{h+1} \right ]_{ij}
$$
$$
\left [D_{\mathrm{Ham}}^h \right ]_{ij} \leq
\left [D_{\mathrm{disc}}^{h} \right ]_{ij}.
$$

An optimal transportation plan $P^{h}$ for $f^{h}_{\scriptscriptstyle{WL}}(G)$ embeddings satisfies
$$
  \left \langle P^{h}, D_{\mathrm{Ham}}^{h}\right \rangle \leq \left \langle P, D_{\mathrm{Ham}}^{h}\right \rangle \; \forall P \in \Gamma,
$$
which can also be written as
$$
\left \langle P^{h}, D_{\mathrm{Ham}}^{h}\right \rangle =
\frac{1}{h} \left ((h-1) \cdot \left \langle P^{h}, D_{\mathrm{Ham}}^{h-1} \right \rangle\\
+ \left \langle P^{h}, D_{\mathrm{disc}}^{h} \right \rangle \right ).
$$
  The values of $D_\mathrm{Ham}^h$ increase in a step-wise fashion for increasing $h$, and their ordering remains constant, except for entries that were $0$ at iteration $h-1$ and became $\frac{1}{h}$ at iteration $h$. Hence, because our metric distance matrices satisfy monotonicity conditions and because $P^h$ is optimal for $D_\mathrm{disc}^h$ according to Lemma \ref{lemma:discrete_distance}, it follows that
$$
\left \langle P^{h}, D_{\mathrm{Ham}}^{h-1}\right \rangle \leq \left \langle P, D_{\mathrm{Ham}}^{h-1}\right \rangle \; \forall P \in \Gamma.
$$
Therefore, $P^h$ is also optimal for $f^{h-1}_{\mathrm{\scriptscriptstyle{WL}}}(G)$ embeddings.
\end{proof}

\subsection{Proof of Theorem \ref{th:neg}}
\label{proof:theorem}
\begin{proof}
Using the same notation as for the Proof in Appendix \ref{proof:discrete_distance} and the formulation in Equation \ref{eq:wassdiscreteformulation}, we can write
 \begin{align*}
D^{f_{\mathrm{\scriptscriptstyle{WL}}}}_W(G,G') &= \min_{P^H \in \Gamma} \left \langle P^H, D_{\mathrm{Ham}}^H\right \rangle \\
 &= \min_{P^H \in \Gamma} \frac{1}{H}\sum_{h=0}^H\langle P^H,  D_{\mathrm{disc}}^h \rangle.
\end{align*}

Let $P^*$ be an optimal solution for iteration $H$. Then, from Lemmas \ref{lemma:discrete_distance} and \ref{lemma:besttransplan}, it is also an optimal solution for $D^H_{\mathrm{disc}}$ and for all $h=0,\ldots,H-1$. We can rewrite the equation as a sum of optimal transport problems:
\begin{equation}
D^{f_{\mathrm{\scriptscriptstyle{WL}}}}_W(G,G') = \frac{1}{H} \sum_{h=0}^H \min_{P^* \in \Gamma} \:\langle P^*, D_{\mathrm{disc}}^h \rangle.
\end{equation}
This corresponds to a sum of 1-dimensional optimal transport problems relying on the discrete metric, which were shown to be conditional negative functions \citep{gardner2017definiteness}. Therefore, the final sum is also conditional negative definite.
\end{proof}

\subsection{Data sets and additional results}
\label{app:additional_ds}

\begin{table*}[b]
\caption{Description of the experimental data sets}
\begin{center}
\begin{sc}
\resizebox{\textwidth}{!}{
    \begin{tabular}{lcccccccccc}
    \toprule
    Data set & Class Ratio  & Node Labels & Node Attributes & Edge Weights & $\#$ Graphs  &  Classes\\
    \midrule
    MUTAG & $63/125$ & \checkmark & - & - & $188$ & $2$\\ 
    NCI1 & $2053/2057$ & \checkmark & - & - & $4110$ & $2$\\
    PTC-MR & $152/192$& \checkmark & - & - & $344$ & $2$\\
    D$\&$D & $487/691$ &\checkmark & - & - & $1178$ & $2$\\
    \midrule 
    ENZYMES & $100$ per class &\checkmark & \checkmark & - & $600$ & $6$\\
    PROTEINS & $450/663$ &\checkmark & \checkmark & - & $1113$ & $2$\\
    \midrule
    BZR & $86/319$ &\checkmark & \checkmark & - & $405$ & $2$\\
    COX2 & $102/365$ &\checkmark & \checkmark & - & $467$ & $2$\\
    SYNTHIE & $100$ per class & - & \checkmark & - & $400$ & $4$ \\
    IMDB-BINARY & $500/500$ & - & (\checkmark) & - & $1000$ & $2$\\
    SYNTHETIC-NEW & $150/150$ & - & \checkmark & - & $300$ & $2$ \\
    \midrule
    BzR-MD & $149/157$ &\checkmark & - & \checkmark & $306$ & $2$\\
    COX2-MD & $148/155$ &\checkmark & - & \checkmark & $303$ & $2$\\
    \bottomrule
    \end{tabular}
}
\end{sc}
\end{center}
\label{table:data}
\vskip -0.1in
\end{table*}
We report additional information on the data sets used in our experimental comparison in Supplementary Table \ref{table:data}.
Our data sets belong to multiple chemoinformatics domains, including small molecules (\textsc{MUTAG}, \textsc{PTC-MR}, \textsc{NCI1}), macromolecules (\textsc{ENZYMES}, \textsc{PROTEINS}, \textsc{D\&D}) and chemical compounds (\textsc{BZR}, \textsc{COX2}). We further consider a movie collaboration data set (\textsc{IMDB}, see~\citep{yanardag2015deep} for a description) and two synthetic data sets \textsc{Synthie} and \textsc{Synthetic-new}, created by~\citet{morris2016faster} and \citet{feragen2013scalable}, respectively.
The \textsc{BZR-MD} and \textsc{COX2-MD} data sets do not have node attributes but contain the atomic distance between each connected atom as an edge weight. We do not consider distances between non-connected nodes \citep{kriege2012subgraph} and we equip the node with one-hot-encoding categorical attributes representing the atom type, i.e., what is originally intended as a categorical node label. On \textsc{IMDB-B}, \textsc{IMDB-BINARY} was used with the node degree as a \mbox{(semi-)continuous} feature for each node~\citep{yanardag2015deep}. For all the other data sets, we use the off-the-shelf version provided by \citet{repoker2016}.

Results on synthetic data sets are provided in Table \ref{app:table_synth}. We decided not to include those in the main manuscript because of the severely unstable and unreliable results we obtained. In particular, for both data sets, there is a high variation among the different methods. Furthermore, we experimentally observed that even a slight modification of the node features (e.g., normalisation or scaling of the embedding scheme) resulted in a large change of performances (up to $15\%$). Additionally, it has been previously reported \citep{morris2016faster,feragen2013scalable} that on \textsc{Synthetic-new}, a WL with degree treated as categorical node label outperforms the competitors, suggesting that the continuous attributes are indeed not informative. Therefore, we excluded these data sets from the main manuscript, as we concluded that they could not fairly assess the quality of our methods.

\begin{table*}[t]
\caption{Classification accuracies on synthetic graphs with continuous node attributes. Comparison of hash graph kernel (HGK-WL, HGK-SP), GraphHopper kernel (GH), and our method (WWL).}
\label{app:table_synth}
\vskip 0.15in
\begin{center}
\begin{small}
\begin{sc}
 \setlength{\tabcolsep}{0.5pt}
    \begin{tabular}{lcccccccc}
    \toprule
    Method & SYNTHIE & SYNTHETIC-new \\
    \midrule
    VH-C & $27.51\pm 0.00$ & $60.60 \pm 1.60$\\
    RBF-WL &  $94.43 \pm 0.55$ & $86.37 \pm 1.37$ \\
    \midrule
    HGK-WL    & $81.94 \pm 0.40$ & $\mathbf{95.96 \pm 0.25^\ast}$   \\
    HGK-SP & $85.82 \pm 0.28$ & $80.43 \pm 0.71$ \\
    \midrule
    GH  &  $83.73 \pm 0.81$ & $88.83 \pm 1.42$ \\
    \midrule
    WWL   & $\mathbf{96.04 \pm 0.48^\ast}$ & $86.77 \pm 0.98$\\
    \bottomrule
    \end{tabular}
\end{sc}
\end{small}
\end{center}
\vskip -0.1in
\end{table*}

\subsection{Details on hyperparameter selection}
\label{app:hypersel}
The following ranges are used for the hyperparameter selection: the parameter of the SVM $C = \{10^{-3},\ldots, 10^{3}\}$ (for continuous attributes) and $C = \{10^{-4},\ldots, 10^{5}\}$ (for categorical attributes); the WL number of iterations $h=\{0,\ldots, 7\}$; the $\lambda$ parameter of the WWL $\lambda=\{10^{-4},\dots,10^{1}\}$. For RBF-WL and VH-C, we use default $\gamma$ parameter for the Gaussian kernel, i.e., $\gamma = 1/m$, where $m$ is the size of node attributes. 
For the GH kernel, we also fix the $\gamma$ parameter to $1/m$. For HGK, we fix the number of iterations to $20$ for each data set, except for SYNTHETICnew where we use $100$ (these setups were suggested by the respective authors~\citep{morris2016faster,feragen2013scalable}. Furthermore, because HGK is a randomised method, we compute each kernel matrix $10$ times and average the results. 
When the dimensionality of the continuous attributes $m>1$, these are normalised to ensure comparability among the different feature scales, in each data set except for BZR and COX2, due to the meaning of the node attributes being location coordinates.

\subsection{Runtime comparison}
\label{app:runtime}

Overall, we note that WL and WL-OA scale linearly with the number of nodes; therefore, these methods are faster than our approach. Because of the differences in programming language implementations of the different methods, it is hard to provide an accurate runtime comparison. However, we empirically observe that the Wasserstein graph kernels are still competitive, and a kernel matrix can be computed in a median time of $~40$ s, depending on the size and number of graphs (see Figure~\ref{fig:runtime}). For the continuous attributes, our approach has a runtime comparable to GH. However, although our approach can benefit from a significant speedup (see discussion below and Section \ref{sec:expsetup}), GH was shown to empirically scale quadratically with the number of graph nodes~\citep{feragen2013scalable}. The HGK, on the other hand, is considerably slower, given the number of iterations and multiple repetitions while taking into account the randomisation.

To evaluate our approach with respect to the size of the graphs and recalling that computing the Wasserstein distance has complexity $\mathcal{O}(n^3 log(n))$, we simulated a fixed number of graphs with a varying average number of nodes per graph. We generated random node embeddings for $100$ graphs, where the number of nodes is taken from a normal distribution centered around the average number of nodes. We then computed the kernel matrix on each set of graphs to compare the runtime of regular Wasserstein with the Sinkhorn regularised optimisation. As shown in Supplementary Figure~\ref{fig:runtime}, the speedup starts to become beneficial at approximately 100 nodes per graph on average, which is larger than the average number of nodes in the benchmark data sets we used.

To ensure good performance when using the Sinkhorn approximation, we evaluate the obtained accuracy of the model. 
Recalling that the Sinkhorn method solves the following entropic regularisation problem,
$$P^\gamma = \argmin_{P \in \Gamma(X,X')} \left \langle P, M \right \rangle - \gamma h(P),$$
we further need to select $\gamma$. Therefore, on top of the cross-validation scheme described above, we further cross-validate over the regularisation parameter values of $\gamma \in \{0.01, 0.05, 0.1, 0.2, 0.5, 1, 10\}$ for the \textsc{Enzymes} data set and obtain an accuracy of $72.08 \pm 0.93$, which remains above the current state of the art. Values of $\gamma$ selected most of the time are $0.3$, $0.5$, and $1$.

\begin{figure*}[t]
\vskip 0.2in
  \centering
  \includegraphics[width=0.9\columnwidth]{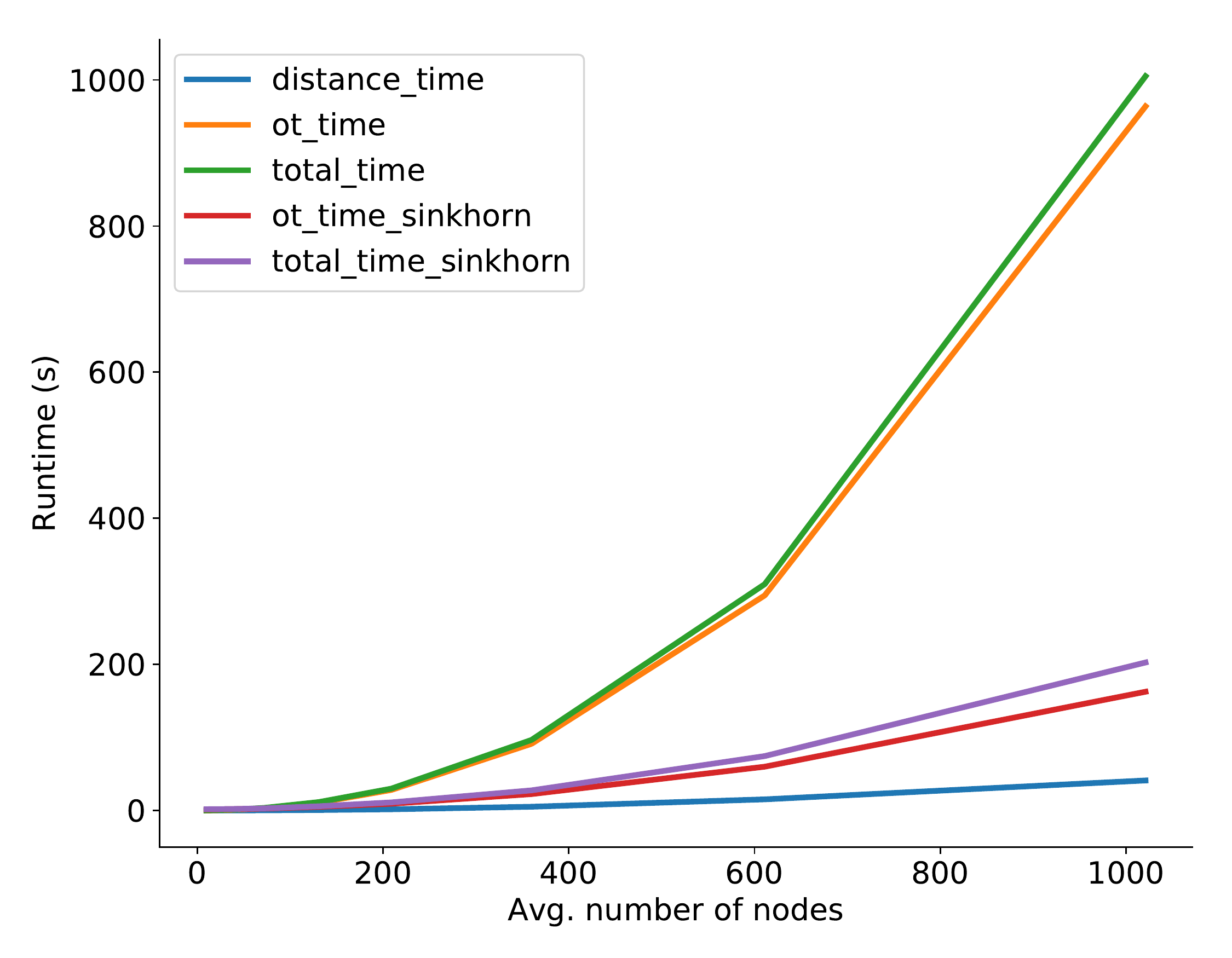}
  \includegraphics[width=0.9\columnwidth]{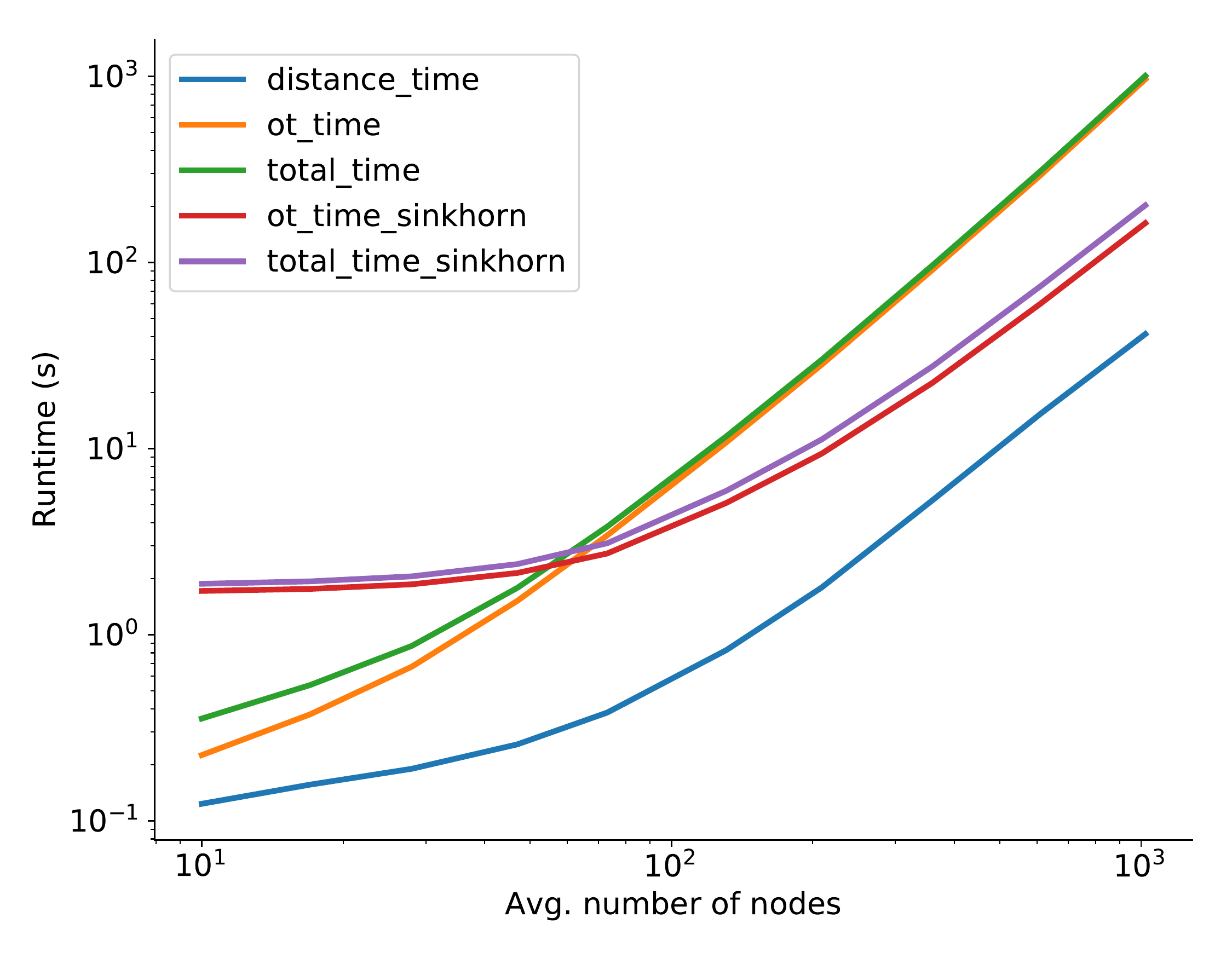}
\caption{Runtime performance of the WWL Kernel computation step with a fixed number of graphs. We also report the time taken to compute the ground distance matrix as $\texttt{distance\_time}$. Here, $\texttt{total\_time}$ is the sum of the time to compute the ground distance and the time taken to solve the optimal transport (ot) problem for the regular solver or the Sinkhorn-regularised one.
The logarithmic scale on the right-side figure shows how, for a small average number of nodes, the overhead to run Sinkhorn is higher than the benefits.}
\label{fig:runtime}
\vskip -0.2in
\end{figure*}

\subsection{Performance on isomorphic synthetic graphs}
\label{app:isomorphic}
We performed an additional experiment to evaluate the difference between WL and WWL for noisy  Erd\H{o}s--R\'enyi graphs~($n = 30$, $p = 0.2$). We report the relative distance between $G$ and its permuted and perturbed variant $G'$, w.r.t.\ a third independent graph $G''$ for an increasing noise level~(i.e., edge removal) in Figure~\ref{fig:ergraph}. We see that WWL is more robust against noise. 

\begin{figure*}[t]
\vskip 0.2in
  \centering
  \includegraphics[width=0.9\columnwidth]{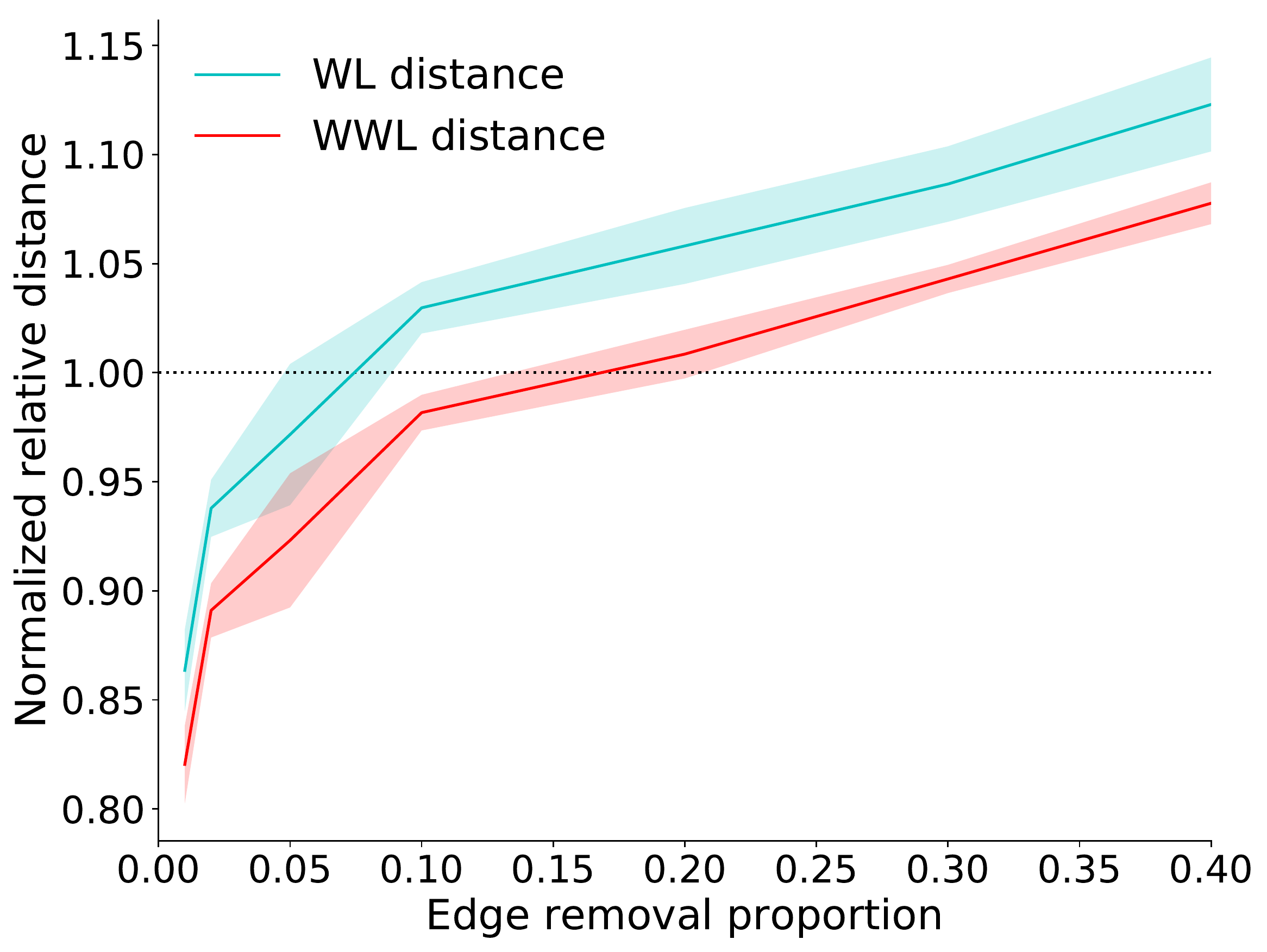}
\caption{Relative distance between (Erd\H{o}s--R\'enyi) graph $G$ and its permuted and perturbed variant $G'$ w.r.t.\ a third independent graph $G''$ for an increasing noise level.}
\label{fig:ergraph}
\vskip -0.2in
\end{figure*}

\end{document}